\documentclass[conference,10pt]{IEEEtran}
\IEEEoverridecommandlockouts
\usepackage{amsmath,amsfonts,amsthm,amssymb,bbm}
\usepackage{cite}

\usepackage{balance}
\usepackage{mathrsfs}
\usepackage{xcolor}
\usepackage{tikz}
\usepackage{pgfplots}
\usepackage{colortbl}
\usepackage{graphicx}
\usepackage{svg}
\graphicspath{figures}
\usetikzlibrary{positioning,quotes,angles,patterns,intersections,pgfplots.fillbetween,calc}
\usepackage{tkz-euclide}
\usepackage{mathtools}
\usepackage{url}
\usetikzlibrary{decorations.pathmorphing}
\usepackage{enumitem}
\usepackage{stfloats}
\pgfplotsset{compat=1.17} 
\usepackage{hyperref}
\usepackage{cleveref}
\usepackage{comment}
\usepackage{multicol}

\def \extended {1}

\usepackage[normalem]{ulem}
\DeclareMathOperator*{\argmin}{arg\,min}

\theoremstyle{plain} 
\newtheorem{theorem}{Theorem}
\newtheorem{corollary}{Corollary}

\newtheorem{lemma}{Lemma}
\newtheorem{prop}{Proposition}
\theoremstyle{definition} \newtheorem{remark}{Remark}
\theoremstyle{definition} 
\usetikzlibrary{calc,shapes.geometric}
\newtheorem{assumption}{Assumption}

\crefname{assumption}{Assumption}{Assumptions}
\crefname{remark}{Remark}{Remarks}
\crefname{prop}{Proposition}{Propositions}
\crefname{theorem}{Theorem}{Theorems}
\crefname{corollary}{Corollary}{Corollaries}
\crefname{section}{Section}{Sections}
\crefname{appendix}{Appendix}{Appendices}
\crefname{equation}{Equation}{Equations}
\crefname{figure}{Figure}{Figures}
\crefname{lemma}{Lemma}{Lemmas}
\crefname{table}{Table}{Tables}
\newcommand{\map}[3]{#1:#2\rightarrow #3}
\newcommand{\snorm}[1]{\Vert #1 \Vert}

\title{Theoretical Guarantees of Data Augmented Last Layer Retraining Methods}

\author{%
  \IEEEauthorblockN{Monica Welfert, Nathan Stromberg and Lalitha Sankar}
  \IEEEauthorblockA{Arizona State University\\
                    Email: \{mwelfert, nstrombe, lsankar\}@asu.edu}
    \thanks{This work is supported in part by NSF grants CIF-1901243, CIF-1815361, CIF-2007688, DMS-2134256, and SCH-2205080.}
}
\def \extended {1} 
\begin{document}
\maketitle

\begin{abstract} 
 Ensuring fair predictions across many distinct subpopulations in the training data can be prohibitive for large models. Recently, simple linear last layer retraining strategies, in combination with data augmentation methods such as upweighting, downsampling and mixup, have been shown to achieve state-of-the-art performance for worst-group accuracy, which quantifies accuracy for the least prevalent subpopulation. For linear last layer retraining and the abovementioned augmentations, we present the optimal worst-group accuracy when modeling the distribution of the latent representations (input to the last layer) as Gaussian for each subpopulation. We evaluate and verify our results for both synthetic and large publicly available datasets.
\end{abstract}

\section{Introduction}
Last layer retraining (LLR) has emerged as a popular method for leveraging representations from large pretrained neural networks and fine-tuning them to locally available data. These methods are significantly inexpensive computationally relative to training the full model, and thus, allow transferring a model to new domains, predicting on \textit{retraining data} with distributional shifts relative to the original, and optimizing for a different metric than that used by the original model. 

In general, training data includes samples from different subpopulations \cite{yang2023change}. 
Assuring fair inferences across all subpopulations remains an important problem in modern machine learning. A metric which has been recently evaluated with good success for assuring fair decisions is  worst-group accuracy (WGA), a worst-case metric for any prior across subpopulations. Existing methods which optimize for WGA utilize strongly regularized models along with \textit{data augmentation} methods such as \textit{downsampling} \cite{kirichenko2022last, labonte2023towards}, \textit{upweighting} \cite{liu2021just, qiu2023simple}, and \textit{mixing} \cite{yaoImprovingOutofDistributionRobustness2022, giannone2021just} (\cref{sec:setup} presents precise definitions of these methods). These {augmentation techniques} help to account for varying proportions of individual subpopulations and enable the final model to predict well on every subpopulation. 

It is difficult to obtain theoretical performance guarantees for large models. However, for a fixed representation-extracting model, one can focus on evaluating LLR techniques that tune a linear last layer using (possibly augmented) representations from the pretrained model. 
We study this setting and model the representations of the subpopulations using tractable distributions; this allows us to directly compare different data augmentation techniques in terms of worst group error ($\text{WGE}=1-\text{WGA}$) and finite sample performance. 

To this end, analogous 
to~\cite{yaoImprovingOutofDistributionRobustness2022}, we model individual sub-populations as distinct Gaussian distributions. Our primary contribution is a straightforward comparison of the three most common data augmentation techniques for WGE: downsampling, upweighting, and mixing. We evaluate the performance of the abovementioned augmentation methods in three ways using: (i) the learned linear models, (ii) the resulting WGE, and (iii) the sample complexity in the setting of a finite number of training examples.
Our key contributions are in providing:
\begin{itemize}[leftmargin=*]
    \item A distribution-free equivalence of the risk minimization problem, and thus the optimal models and performance, for upweighting and downsampling (\cref{thm:ds-uw-eq}). To the best of our knowledge, this is a new result.
    \item Statistical analysis of the WGE for each data augmentation method under Gaussian subpopulations (\cref{thm:wge-comparison}).
    \item Sample complexity of each method (\cref{thm:sample_complexity}).
    \item Empirical results that match theory for Gaussian mixtures and the CMNIST, CelebA and Waterbirds datasets.
\end{itemize}
Our work is distinct from that in \cite{yaoImprovingOutofDistributionRobustness2022} as follows: (i) explicit incorporation of the minority group priors; (ii) providing precise WGE guarantees (in contrast to bounds in \cite{yaoImprovingOutofDistributionRobustness2022}); and (iii) including downsampling and upweighting as data augmentation methods (the focus in \cite{yaoImprovingOutofDistributionRobustness2022} is primarily on mixing and its variants) for which we also provide comparative model and error guarantees beyond the Gaussian setting.  
\section{Problem Setup}
\label{sec:setup}
We consider the supervised classification setting and assume that the LLR methods have access to a representation of the input/\textit{ambient} (original high-dimensional data such as images etc.) data, the ground-truth label, as well as the domain annotation. Taken together, the label and domain combine to define the group annotation for any sample. For ease of analysis, we assume binary labels (belonging to $\{0,1\}$) and binary domains (belonging to $\{S,T\}$). More formally, the training
dataset is a collection of i.i.d. tuples of the random variables $(X_a,Y,D) \sim P_{X_a,Y,D}$,  
where $X_a \in \mathcal{X}_a$ is the ambient high-dimensional sample, $Y \in \mathcal{Y} = \{0,1\}$ is the class label, and  $D \in \mathcal{D}=\{S,T\}$ is the domain label. Since the focus here is on learning the linear last layer, we denote the \textit{latent} representation that acts as an input to this last layer by $X\coloneqq \phi(X_a)$ for an embedding function $\map{\phi}{\mathcal{X}_a}{\mathcal{X}\subseteq\mathbb{R}^p}$ such that the training dataset for LLR is $(X,Y,D)\sim P_{X,Y,D}$. 

The tuples $(Y,D)$ of class and domain labels partition the examples into four different groups. Let $\pi^{(y,d)}\coloneqq P(Y=y,D=d)$ for $(y,d) \in \mathcal{Y}\times\mathcal{D}$.
We denote the linear correction applied in the latent space of a pretrained model as $\map{f_\theta}{\mathcal{X}}{\mathbb{R}}$, which is parameterized by 
a linear decision boundary $\theta=(w,b)\in\mathbb{R}^{p+1}$ given by 
\begin{equation}
    f_\theta(x) = w^Tx + b.
\end{equation}
The statistically optimal linear model is obtained by minimizing the risk defined as
\begin{equation}
    R(f_\theta) \coloneqq \mathbb{E}_{P_{X,Y,D}}[\ell(f_\theta(X),Y)],
    \label{eq:risk}
\end{equation}
where $\ell:\mathbb R\times\mathcal{Y} \to \mathbb R_+$ is a loss function.
We consider four different methods to learn a classifier: (a) {standard risk minimization} (SRM), (b) downsampling (DS), (c) upweighting (UW), and (d) intra-class domain mixup (MU). In particular, SRM involves minimizing \eqref{eq:risk} as is whereas downsampling involves reducing the size of each group to that of the smallest one while upweighting involves scaling the loss for each group in proportion to the inverse of the prior.  Finally, intra-class domain mixup takes an arbitrary convex combination of two randomly sampled representations from the same class but from different domains. 

A general formulation for obtaining the optimal $f_{\theta^*}$ is: 
\begin{equation}
    \theta^* = \argmin_\theta \mathbb{E}_{P_{X,Y,D}}[\ell(f_\theta(X),Y)c(Y,D)], \; 
    \label{eq:gen-opt}
\end{equation}
where $c(y,d)=1$, $(y,d)\in\mathcal{Y}\times\mathcal{D}$, for SRM, DS, and MU, but $c(y,d) = 1/({4\pi^{(y,d)}})$ for UW. Moreover, the priors on the groups remain the same as the true statistics, and therefore SRM, for all methods except DS where $\pi^{(y,d)} = {1}/{4}$. Finally, for MU, the representation $X$ is now $X=\Lambda X_1 + (1-\Lambda)X_2$ where $X_1 \sim P_{X|Y=y,D=S}$, $X_2 \sim P_{X|Y=y,D=T}$, $y \in \mathcal{Y}$, and the mixup parameter $\Lambda \sim \text{Beta}(\alpha,\alpha)$.


We desire a model that makes fair decisions across groups, and therefore, we
evaluate worst-group error, i.e., the maximum error among all groups, defined for a model $f_\theta$ as
\begin{equation}
    \text{WGE}(f_\theta) \coloneqq \max_{(y,d)\in\mathcal{Y}\times\mathcal{D}} E^{(y,d)}(f_\theta),
    \label{eq:wge}
\end{equation}
where $E^{(y,d)}(f_\theta)$ denotes the per-group misclassification error for $(y,d) \in \mathcal{Y}\times\mathcal{D}$. Specifically, for $(y,d) \in \mathcal{Y}\times\mathcal{D}$:
\begin{equation}
    E^{(y,d)}(f_\theta) \coloneqq P(\mathbbm{1}\{f_\theta(X) > 1/2\} \ne Y | Y=y,D=d)
    \label{eq:group-error}
\end{equation}
where the threshold 1/2 is chosen to match $Y\in\{0,1\}$. 

\section{Main Results}\label{section:mainresults}
Our first result observes that, for any chosen loss, UW and DS yield the same statistically expected predictor. We collate the proofs in the Appendix and outline a proof sketch here. 
\begin{theorem}
    For any given $P_{X,Y,D}$ and loss $\ell$, the objectives in \eqref{eq:gen-opt} when modified appropriately for DS and UW are the same. Therefore, if a minimizer exists for one of them, then the minimizer of the other is the same, i.e., $\theta^*_\text{DS}=\theta^*_\text{UW}$.
    \label{thm:ds-uw-eq}
\end{theorem}
\begin{proof}[Proof sketch]\let\qed\relax The key intuition here is that the upweighting factor is proportional to the inverse of the priors on each group. Thus, when the expected loss is decomposed into an expectation over groups, the priors from the expected loss cancel and we recover the downsampled problem. A detailed proof can be found in 
\if \extended 0%
~\cite[Appendix~A]{KurriWSS21}.
\fi%
\if \extended 1%
\cref{appendix:thm1-proof}.
\fi%
\end{proof}
\begin{remark}
    Although we are focused on the binary class and domain label setting, \cref{thm:ds-uw-eq} holds for any number of classes and domains by replacing $\pi^{(y,d)}=1/n_g$ and $c(y,d)=1/(n_g\pi^{(y,d)})$ for $(y,d) \in \mathcal{Y}\times\mathcal{D}$, where $n_g$ is the number of groups. To the best of our knowledge, such an analysis, albeit simple, has not been presented before. 
\end{remark}
While \cref{thm:ds-uw-eq} holds for any general data distribution, to obtain more refined guarantees on WGE and model parameters for different augmentation methods considered here, we make the following tractable assumptions on the dataset. Such assumptions have recently been introduced for tractability in the analysis of out-of-distribution robustness  (e.g., \cite{yaoImprovingOutofDistributionRobustness2022}). 
\begin{figure}
    \centering
    \includegraphics[width=0.95\linewidth]{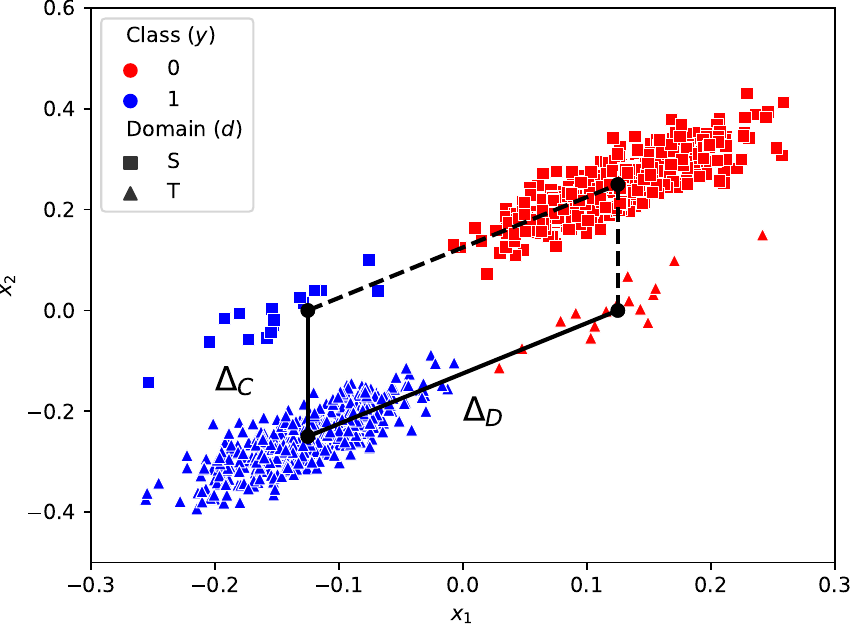}
    \caption{$\Delta_C$ and $\Delta_D$ are shown as line segments between group means overlaid on data sampled from Gaussian mixtures satisfying \cref{as:latent_normal,as:equal_priors,as:mean_difference,as:orthogonality}.}
    \label{fig:data}
\end{figure}
\begin{assumption}
    \label{as:latent_normal} $X \in \mathcal{X}$ is distributed according to the following mixture of Gaussians: 
    \begin{align}
        X | (Y=y, D=d) &\sim \mathcal{N}(\mu^{(y,d)}, \Sigma),
        \label{eq:gauss-model}
    \end{align}
    for $(y,d) \in \mathcal{Y}\times\mathcal{D}$, where $\mu^{(y,d)}\coloneqq \mathbb E[X|Y=y,D=d] \in \mathbb R^p$ and $\Sigma \in \mathbb R^{p\times p}$ is symmetric positive definite. Additionally, we place priors $\pi^{(y,d)}$, $(y,d) \in \mathcal{Y}\times\mathcal{D}$, on each group and priors $\pi^{(y)}\coloneqq P(Y=y)$, $y \in \mathcal{Y}$, on each class. 
\end{assumption}
\begin{assumption}
    \label{as:equal_priors} The minority groups have equal priors, i.e., for $\pi_0\le\frac{1}{4}$, 
    \begin{align*}
        \pi^{(0,T)} = \pi^{(1,S)} = \pi_0 \quad \text{and} \quad
        \pi^{(1,T)} = \pi^{(0,S)} = 1/2-\pi_0.
    \end{align*}
    Also, the class priors are equal, i.e., $\pi^{(0)} = \pi^{(1)} = 1/2$.
\end{assumption}
\begin{assumption}
    \label{as:mean_difference} The difference in means between classes within a domain $\Delta_D \coloneqq \mu^{(1,d)} - \mu^{(0,d)}$ is constant for $d\in\mathcal{D}$.
\end{assumption}
\begin{remark}
     \cref{as:mean_difference} also implies that the difference in means between domains within the same class $\Delta_C \coloneqq \mu^{(y,S)} - \mu^{(y,T)}$ is also constant for each $y\in\mathcal{Y}$. We see this by noting that each group mean makes up the vertex of a parallelogram, as shown in \cref{fig:data}, where $\Delta_D$ and $\Delta_C$ are shown on samples drawn from a distribution satisfying \cref{as:equal_priors,as:latent_normal,as:mean_difference}.
\end{remark}

\begin{prop}
Let $\ell(\hat{y},y) = \|y-\hat{y}\|_2^2$ for $\hat{y} \in \mathbb R$ and $y \in \mathcal{Y} $ be the mean-squared error (MSE) loss. Under \cref{as:latent_normal,as:equal_priors,as:mean_difference}, the minimizers in \eqref{eq:gen-opt} for DS and UW are the same, i.e.,
\[\theta^*_\text{DS} = \theta^*_\text{UW},\]
\begin{equation}
    w^*_\text{DS} = w^*_\text{UW} = \frac{1}{4}\left(\Sigma + \frac{1}{4}\Delta_C \Delta_C^T + \frac{1}{4}\Delta_D \Delta_D^T\right)^{-1}\Delta_D
    \label{eq:opt-w-ds-1}
\end{equation}
\begin{equation}
 b^*_\text{DS} = b^*_\text{UW} =\frac{1}{2} - \frac{1}{2}(w^*_\text{DS})^T\left(\mu^{(0,T)}+\mu^{(1,S)}\right),
 \label{eq:opt-b-ds-1}
\end{equation}
and thus, $
    \text{WGE}(f_{\theta^*_\text{DS}}) = \text{WGE}(f_{\theta^*_\text{UW}})$. 
\label{prop:ds-uw-weights-eq}
\end{prop}
\begin{proof}[Proof sketch]\let\qed\relax
The proof of parameter equality follows directly from \cref{thm:ds-uw-eq}. To obtain the specific forms of the parameters, we derive the optimal parameters in \eqref{eq:gen-opt} for the given $\ell$ and appropriate values of $c(y,d)$, $(y,d) \in \{0,1\} \times \{S,T\}$, for DS and UW. We then use \cref{as:latent_normal} to obtain the WGE in terms of Gaussian CDFs (as detailed in \cref{appendix:ds-uw-weights-eq}).
\end{proof}

Note that if $\ell(\hat{y},y) = \|y-\hat{y}\|_2^2 + \lambda \|w\|_1$ for $\hat{y} \in \mathbb R$, $y \in \mathcal{Y}$, and a regularizer $\lambda >0$, then \eqref{eq:gen-opt} simplifies to the deep feature reweighting (DFR) optimization, an $\ell_1$-regularized DS method that achieves state of the art WGA for many datasets \cite{kirichenko2022last}.
\begin{corollary}
Let $\ell(\hat{y},y) = \|y-\hat{y}\|_2^2 + \lambda \|w\|_1$ for $\hat{y} \in \mathbb R$, $y \in \mathcal{Y}$, and $\lambda >0$. Under \cref{as:latent_normal,as:equal_priors,as:mean_difference},the minimizer in \eqref{eq:gen-opt} for DFR is the same as that for UW, i.e.,
\[\theta^*_\text{DFR} = \theta^*_\text{UW}.\]
Thus,
\[\text{WGE}(f_{\theta^*_\text{DFR}}) = \text{WGE}(f_{\theta^*_\text{UW}}).\]
\label{cor:dfr-uw-weights-eq}
\vspace{-20pt}
\end{corollary}
The proof of \cref{cor:dfr-uw-weights-eq} follows from \cref{thm:ds-uw-eq} and \cref{prop:ds-uw-weights-eq}.

As derived in the proof of \cref{prop:ds-uw-weights-eq}, the WGEs result from computing Gaussian CDFs at the optimal model for each method. However, while \cref{prop:ds-uw-weights-eq} clarifies the statistical behavior of DS and UW, comparing the resulting analytical expressions for WGEs for each of the four methods requires finer assumptions. To this end, we make the following orthogonality assumption.

\begin{assumption}
    $\Delta_D$ and $\Delta_C$ are orthogonal w.r.t. the $\Sigma^{-1}$--inner product, i.e., $\Delta_C^T \Sigma^{-1} \Delta_D = 0$.
    \label{as:orthogonality}
\end{assumption}

\begin{theorem}
Let $\ell(\hat{y},y) = \|y-\hat{y}\|_2^2$, $\hat{y} \in \mathbb R$, $y \in \mathcal{Y} $. For \cref{as:latent_normal,as:equal_priors,as:mean_difference}, the optimal SRM and MU models are: 
\begin{equation}
    w^*_\text{SRM} = \frac{1}{4}\left(\Sigma + 2\pi_0(1-2\pi_0)\Delta_C\Delta_C^T + \frac{1}{4}\bar\Delta\bar\Delta^T \right)^{-1}\bar\Delta,
    \label{eq:opt-w-srm-1}
\end{equation}
where $\bar\Delta \coloneqq \mu^{(1)}-\mu^{(0)} = \Delta_D - (1-4\pi_0)\Delta_C$, and
\begin{equation}
    b^*_\text{SRM} = \frac{1}{2} - \frac{1}{2}(w^*_\text{SRM})^T(\mu^{(0,T)}+\mu^{(1,S)}),
    \label{eq:opt-b-srm-1}
\end{equation}
\begin{equation}
    w^*_\text{MU} = \frac{1}{4}\left(2\mathbb E[\Lambda^2]\Sigma + \text{Var}(\Lambda)\Delta_C\Delta_C^T + \frac{1}{4}\Delta_D\Delta_D^T \right)^{-1}\Delta_D,
    \label{eq:opt-w-mu-1}
\end{equation}
\begin{equation}
    b^*_\text{MU} = \frac{1}{2} - \frac{1}{2}(w^*_\text{MU})^T(\mu^{(0,T)}+\mu^{(1,S)}).
    \label{eq:opt-b-mu-1}
\end{equation}
Additionally, under \cref{as:orthogonality} and for $\pi_0 < 1/4$,
\begin{align*}
        \text{WGE}(f_{\theta^*_\text{SRM}}) > \text{WGE}(f_{\theta^*_\text{DS}}) = \text{WGE}(f_{\theta^*_\text{UW}}) = \text{WGE}(f_{\theta^*_\text{MU}}).
    \end{align*}
\label{thm:wge-comparison}
\vspace{-20pt}
\end{theorem}
\begin{proof}[Proof sketch]\let\qed\relax
The proof follows similarly to that of \cref{prop:ds-uw-weights-eq} using the appropriate values of $c(y,d)$, $(y,d) \in \{0,1\} \times \{S,T\}$, for SRM and MU. We employ a derivative analysis to show $\text{WGE}(f_{\theta^*_\text{SRM}}) > \text{WGE}(f_{\theta^*_\text{DS}})$ and then show the remaining equalities. See \cref{fig:optimal_lines} for a plot showing the optimal planes for each method for data satisfying \cref{as:latent_normal,as:equal_priors,as:mean_difference,as:orthogonality}.   See \cref{appendix:wge-comparison} for a detailed proof.
\end{proof}

In practice, we only have access to a finite number of samples. In this setting, we can approximate the risk in \eqref{eq:risk} by the empirical risk, defined for a given dataset of $n$ samples $(x_i, y_i, d_i) \in \mathcal{X}\times\mathcal{Y}\times\mathcal{D}$, $i = 1,\dots,n$, drawn i.i.d. from $P_{X,Y,D}$ and a loss $\ell$ as 
\begin{equation}
    \hat{R}(f_\theta) \coloneqq \frac{1}{n}\sum_{i=1}^n \ell(f_\theta(x_i),y_i).
    \label{eq:emp-risk}
\end{equation}
We consider the same four methods as before where SRM is now just the empirical risk minimization (ERM). The empirically optimal $f_{\hat{\theta}}=\hat{w}^Tx+\hat{b}$ is obtained from  
\begin{equation}
    \hat{\theta} = \argmin_\theta \frac{1}{n}\sum_{i=1}^n \ell(f_\theta(x_i),y_i)c(y_i,d_i),
    \label{eq:emp-min}
\end{equation}
where again $c(y,d)=1$, $(y,d)\in\mathcal{Y}\times\mathcal{D}$ for ERM, DS, and MU, but $c(y,d) = {n}/({4n^{(y,d)}})$ for UW with $n^{(y,d)}$ being the number of samples in the group $(y,d)$. In the case of DS, rather than using $n$ samples, we use $4n_\text{min}\coloneqq\min_{(y,d)} n^{(y,d)}$ samples. Finally, for MU, we use $x_i=\lambda_i x_{i_1} + (1-\lambda_i)x_{i_2}$ where $i_1$ and $i_2$ are uniformly chosen from the indices of samples in the groups $(y_i, G)$ and $(y_i, R)$, respectively, and the mixup parameter $\lambda_i \sim \text{Beta}(\alpha,\alpha)$.
The following result compares the sample complexity of each of the four methods. We use the notation $O_p(\cdot)$ for the stochastic boundedness of a sequence of random variables \cite{bishop2007discrete}. More formally, for a sequence of random variables $X_n$ and a sequence of positive scalars $a_n$, $\|X_n\|_2 = O_p(a_n)$ if for any $\varepsilon>0$ there exist finite $M>0$ and $N>0$ such that
\begin{equation}
    P\left(\|{X_n}/{a_n}\|_2 \le M \right) \ge 1-\varepsilon, \quad \text{for all } n > N.
\end{equation}
\begin{theorem}
    Let $\ell(\hat{y},y) = \|y-\hat{y}\|_2^2$, $\hat{y} \in \mathbb R$, $y \in \mathcal{Y} $. Consider $n$ i.i.d. samples generated according to \eqref{eq:gauss-model}, with $n_\text{min}$ being the number of samples in the minority groups. Then
    \begin{align*}
        \| \theta^*_\text{ERM} - \hat{\theta}_\text{ERM}\|_2^2 = \| \theta^*_\text{UW} - \hat{\theta}_\text{UW}\|_2^2 =O_p(p/n),
    \end{align*}
    \[\| \theta^*_\text{DS} - \hat{\theta}_\text{DS}\|_2^2 =O_p(p/n_\text{min}),\]
    and
    \[\| \theta^*_\text{MU} - \hat{\theta}_\text{MU}\|_2^2 =O_p(p\log(n)/n + p/n_\text{min}).\]
    \label{thm:sample_complexity}
    \vspace{-20pt}
\end{theorem}
\begin{proof}[Proof sketch]\let\qed\relax The sample complexity bound for ERM follows from a standard application of the weak law of large numbers \cite{KRIKHELI20217955}. Furthermore, the bound for DS is obtained by setting $n=4n_\text{min}$ while that for UW is a straightforward generalization of that for ERM to the weighted least squares setting. The bound for MU is from \cite{yaoImprovingOutofDistributionRobustness2022}. 
\end{proof}

\section{Experimental Results}
We present numerical results for both synthetic and real-world data for all the augmentation techniques and ERM. 

\subsection{Orthogonal Latent Gaussians}
\begin{figure}
    \centering
    \includegraphics[width=0.95\linewidth]{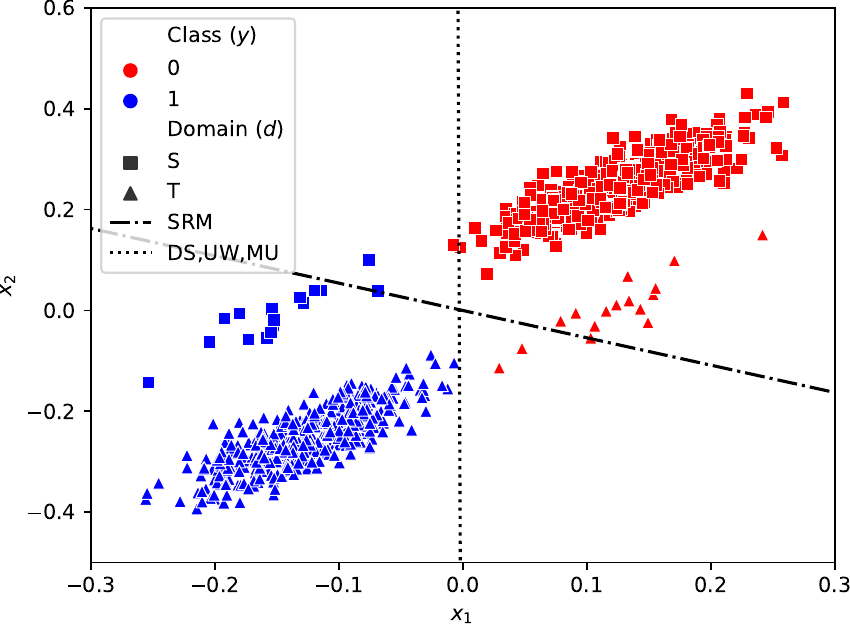}
    \caption{The optimal prediction planes for DS, UW, MU, and SRM are shown overlaid on data sampled from Gaussian mixtures satisfying \cref{as:latent_normal,as:equal_priors,as:mean_difference,as:orthogonality}. The SRM model largely ignores the minority group for each class. }
    \label{fig:optimal_lines}
\end{figure}
We first examine a numerical analog to the mixture Gaussian model given in \cref{as:latent_normal,as:equal_priors,as:mean_difference,as:orthogonality} to empirically study the convergence of DS, UW, and MU methods in terms of MSE from the corresponding statistical solutions. We generate $n$ data points and calculate the empirical weights for each method by performing the corresponding data augmentation and then computing the sample variance (of $X$) and covariance (of $X,Y$) matrices used in the closed-form solution to \eqref{eq:emp-min} with $\ell(\hat{y},y) = \|y-\hat{y}\|_2^2$, $\hat{y} \in \mathbb R$, $y \in \mathcal{Y} $. This training step is repeated 10 times to account for randomness introduced by DS and MU. Furthermore, we average over 10 runs (data generation and training) for different random seeds to account for randomness in the training data. We average over these runs when reporting statistics.

We generate group-conditional Gaussian data with the following parameters satisfying \cref{as:equal_priors,as:latent_normal,as:mean_difference,as:orthogonality}: 
\begin{align*}
    \Delta_C &= \begin{pmatrix}
        0 & \frac{1}{4}
    \end{pmatrix}^T &\quad
    \Delta_D &= \begin{pmatrix}
        -\frac{1}{4} & -\frac{1}{4}
    \end{pmatrix}^T \\
    \Sigma &= \begin{pmatrix}
        .002 & .002 \\
        .002 & .003 \\
    \end{pmatrix} &\quad
    \pi_0 &= \frac{1}{64}.
\end{align*}

 In \cref{fig:wge_gaussian}, we compare the WGE of each training method as a function of the number of samples. We see that each of the data augmentation methods outperforms (non-augmented vanilla) ERM, and they all converge to the same WGE. The equivalence of these methods for large $n$ is implied by \cref{thm:wge-comparison}. However,  we see interesting behavior at small $n$ in \cref{fig:wge_gaussian_zoom}: DS achieves worse WGE than UW or MU at very small $n$. This may be explained by the fact that DS often throws away data while MU and UW keep all available data. Therefore, DS may not be well-suited to limited data regimes.

 \begin{figure}
    \centering
    \includegraphics[width=0.95\linewidth]{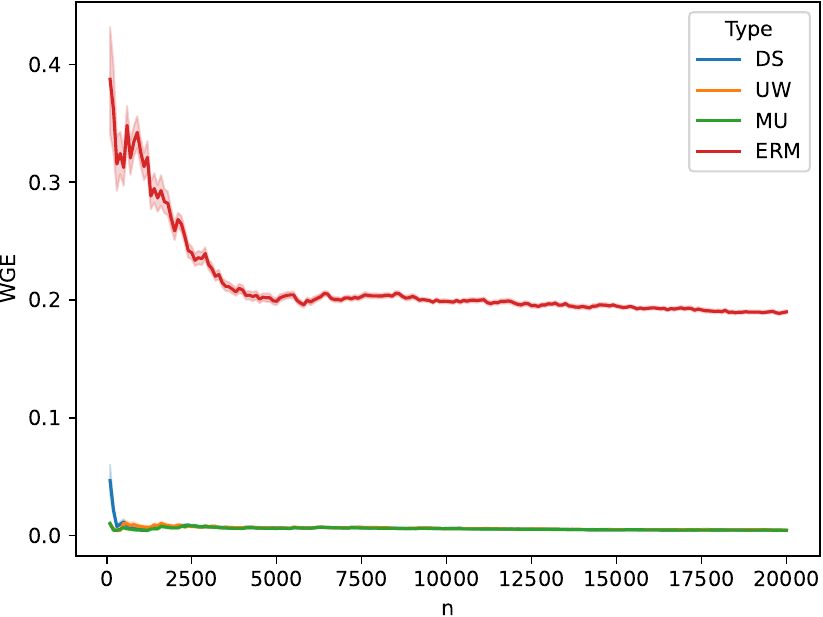}
    \caption{WGE for UW, DS, MU and ERM for the data in \cref{fig:optimal_lines}. As the number of samples $n$ increases, UW, DS, and MU perform better than ERM.}
    \label{fig:wge_gaussian}
\end{figure}

\begin{figure}
    \centering
    \includegraphics[width=0.95\linewidth]{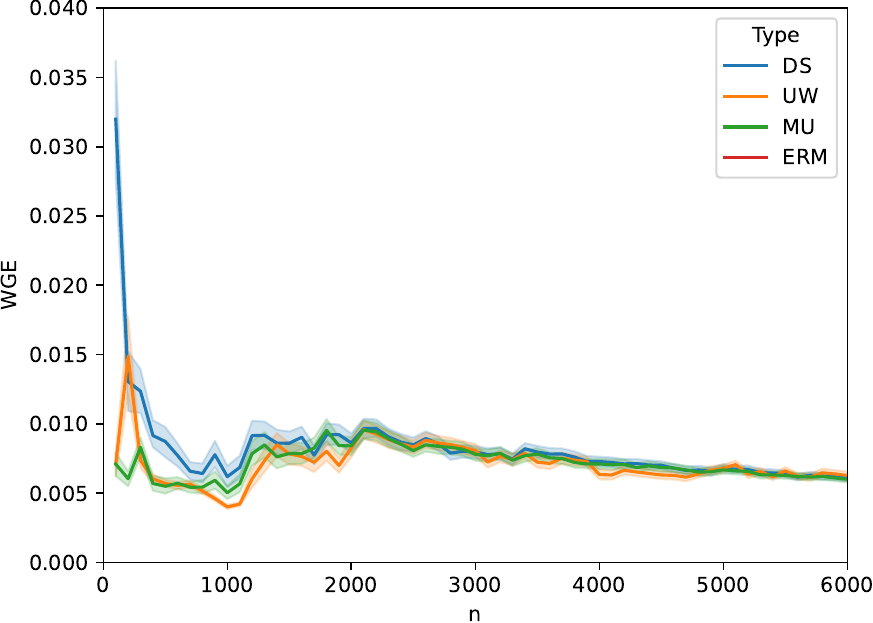}
    \caption{Zoomed in version of \cref{fig:wge_gaussian} where we see the differences between data augmentation methods, especially for small $n$.}
    \label{fig:wge_gaussian_zoom}
\end{figure}

We next compare the empirical weights and bias obtained by each method to the corresponding statistically optimal weights and bias as calculated in \eqref{eq:opt-w-ds-1}, \eqref{eq:opt-b-ds-1} for DS and UW, in \eqref{eq:opt-w-srm-1}, \eqref{eq:opt-b-srm-1}, for SRM, and in \eqref{eq:opt-w-mu-1}, \eqref{eq:opt-b-mu-1} for MU. We report the MSE as a function of $n$ in \cref{fig:mse_gaussian} and compare our results to the bounds found in \cref{thm:sample_complexity}. We see that each method converges at a similar rate, suggesting that tighter sample complexity analysis may be possible.

Finally we demonstrate that each of the data augmentation methods is robust to the prevalence of the minority group. For a fixed $n=10,000$, we train each method with varying $\pi_0$. We see in \cref{fig:gaussian_wge_pi} that the data augmentation methods are robust to even very small minority groups. We additionally note that the performance of ERM approaches that of the data augmented methods as $\pi_0 \rightarrow 1/4$, i.e., the prior for a group-balanced dataset.

\begin{figure}
    \centering
    \includegraphics[width=0.95\linewidth]{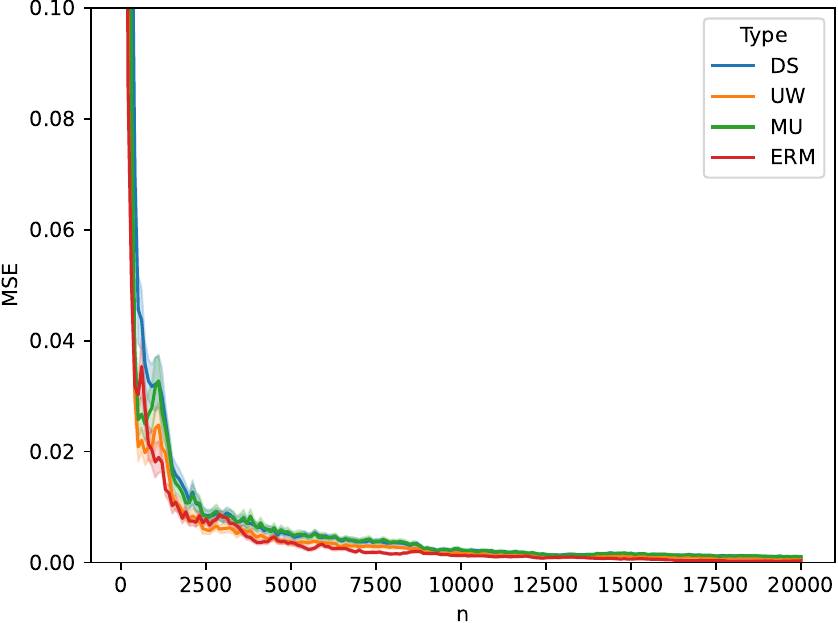}
    \caption{Mean squared error of the estimated weights from data as compared to the expected weights. We see that each method converges quickly to the expected weights as a function of $n$.}
    \label{fig:mse_gaussian}
\end{figure}
\begin{figure}
    \centering
    \includegraphics[width=0.95\linewidth]{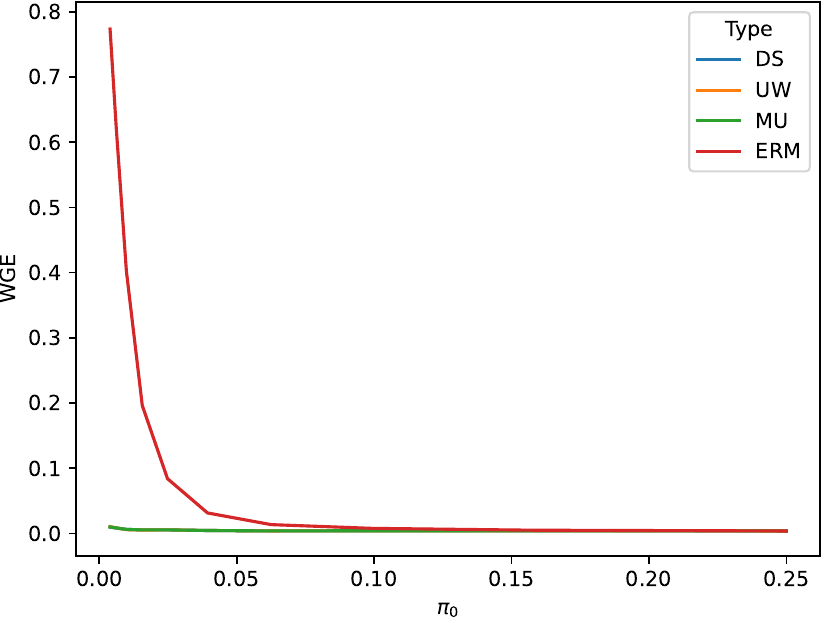}
    \caption{Worst-group error on latent Gaussian subpopulations for each data augmentation technique as a function of $\pi_0$, the prevalence of the minority groups. We see that as $\pi_0$ approaches $1/4$ (a balanced dataset), the WGE of vanilla ERM decreases to match that of the data augmented methods.}
    \label{fig:gaussian_wge_pi}
\end{figure}
\subsection{Publicly Available Large Datasets}
We next consider the CMNIST \cite{arjovsky2019invariant}, CelebA \cite{liu2015faceattributes}, and Waterbirds \cite{sagawa2019distributionally} datasets, which are oft-used in LLR \cite{yang2023change}. CMINST \cite{arjovsky2019invariant} is a variant of the MNIST handwritten digit dataset in which digits 0-4 are labeled $y=0$ and digits 5-9 are labeled $y=1$. The domain is given by color: 90\% of digits labeled $y=0$ are colored green and 10\% are colored red and vice-versa for those labeled $y=1$. 

CelebA \cite{liu2015faceattributes} is a dataset of celebrity faces. We predict hair color as either blonde ($y=1$) or non-blonde ($y=0$), while the domain label is either male ($d=1$) or female ($d=0$). There is a naturally induced correlation between hair color and gender in the dataset due to the prevalence of blonde females. 

Waterbirds \cite{sagawa2019distributionally} is a semi-synthetic image dataset comprised of land birds ($y=1$) or sea birds ($y=0$) on  land ($d=1$) or sea backgrounds ($d=0$). There is a correlation between background and bird type in the training data (sea birds being more present with sea backgrounds) but this correlation is absent in the group- and class-balanced validation data. 

Each dataset is broken into training, validation, and test data. The training data is used to train a large model (ResNet-50 architecture) from which we extract the embedding function $\phi(\cdot)$ used to obtain the latent representations. We view the validation data as a retraining dataset whose representations are used to retrain the last layer of the pretrained model. 

In practice, state-of-the-art methods do not employ the MSE loss. Instead, common methods such as DFR \cite{kirichenko2022last} use highly regularized losses such as log loss with $\ell_1$ penalty. We proceed following this example, and train logistic models with strong $\ell_1$ regularization.

For each of these datasets, we see in \cref{tab:real_world} that all of the data augmentation methods perform similarly and outperform ERM alone. This suggests that the analysis provided here may hold more generally than just on latent Gaussian subpopulations. 
We see that UW and ERM have no variance over runs which is due to the fact that both are deterministic methods, whereas DS and MU introduce randomness. Additionally, these results suggest that DS -- the most common data augmentation method for WGA -- may not have strong advantages over UW or MU, which have advantages in variance, though not in computational complexity. 

\begin{table}
\centering
\caption{WGE (lower is better) Mean $\pm$ StDev \\(averaged over 10 runs)} 
\label{tab:real_world}
\begin{tabular}{llll}
\hline
 & CMNIST & CelebA & Waterbirds \\ \hline
DS     &   7.0 $\pm$ 0.4   &   19.3 $\pm$  3.1  &     9.9 $\pm$  0.8    \\ \hline
UW     &    5.4 $\pm$ 0.0  &   21.7 $\pm$  0.0  &      10.0  $\pm$   0.0 \\ \hline
MU     &    6.2  $\pm$ 0.4 &    22.9 $\pm$ 1.5  &     10.0   $\pm$  0.7 \\ \hline
ERM     &    9.1  $\pm$ 0.0 &    56.7 $\pm$  0.0  &     14.5   $\pm$  0.0 \\ \hline
\end{tabular}
\end{table}

\section{Conclusion}
We have presented a new result that the well-known  data augmentation techniques of DS and UW have statistically identical performance. For LLR, when the latent representations that are input to the last layer are modeled as Gaussian mixtures, MU also achieves the same statistical worst-group accuracy as DS and UW, all of which are better than SRM. Our results are validated for a synthetic Gaussian mixture dataset and appear to hold for several large publicly available datasets. A natural extension is to obtain more refined sample complexity, or equivalently excess risk bounds, when explicitly accounting for the size of each group/subpopulation. An equally compelling question to address is characterizing the finite sample differences between UW and DS for a larger class of distributions building upon the work in \cite{Chaudhuri_Ahuja_Arjovsky_Lopez-Paz_2022}.

\newpage
\bibliographystyle{IEEEtran}
\bibliography{gen}
\if \extended 1
\onecolumn
\appendix
\subsection{Proof of \cref{thm:ds-uw-eq}}
\label{appendix:thm1-proof}
\noindent Expanding the objective of \eqref{eq:gen-opt} for DS, we obtain
\begin{align}
    \mathbb{E}_{P_{X,Y,D}}[\ell(f(X),Y)] &= \mathbb{E}_{P_{Y,D}}\mathbb{E}_{P_{X|Y,D}}[\ell(f(X),Y)|Y,D] \nonumber\\
    & = \sum_{(y,d) \in \mathcal{Y}\times\mathcal{D}} \pi^{(y,d)} \mathbb{E}_{P_{X|Y,D}}[\ell(f(X),Y)|Y=y,D=d] \nonumber \\
    & = \frac{1}{4}\sum_{(y,d) \in \mathcal{Y}\times\mathcal{D}} \mathbb{E}_{P_{X|Y,D}}[\ell(f(X),Y)|Y=y,D=d].
    \label{eq:ds-obj-expanded}
\end{align}
Expanding the objective of \eqref{eq:gen-opt} for UW, we obtain
\begin{align}
    \mathbb{E}_{P_{X,Y,D}}[\ell(f(X),Y)c(Y,D)] &= \mathbb{E}_{P_{Y,D}}\left[c(Y,D)\mathbb{E}_{P_{X|Y,D}}[\ell(f(X),Y)|Y,D]\right] \nonumber\\
    & = \sum_{(y,d) \in \mathcal{Y}\times\mathcal{D}} \frac{\pi^{(y,d)} }{4\pi^{(y,d)}}\mathbb{E}_{P_{X|Y,D}}[\ell(f(X),Y)|Y=y,D=d] \nonumber \\
    & = \frac{1}{4}\sum_{(y,d) \in \mathcal{Y}\times\mathcal{D}} \mathbb{E}_{P_{X|Y,D}}[\ell(f(X),Y)|Y=y,D=d].
    \label{eq:uw-obj-expanded}
\end{align}
Therefore, the objectives of DS and UW are the same, and as a result so are the corresponding minimizers, if they exist.

\subsection{Proof of \cref{prop:ds-uw-weights-eq}}
\label{appendix:ds-uw-weights-eq}
\noindent Let $L(w,b) = \mathbb E_{P_{X,Y,D}}[(Y-w^TX-b)^2c(Y,D)]$.
In order to solve \eqref{eq:gen-opt}, we begin by solving $\partial L(w,b)/\partial b = 0$:
\begin{align*}
    \frac{\partial L(w,b)}{\partial b} &= -2\mathbb E[(Y-w^TX-b)c(Y,D)] = 0  \\
    & \implies b^* = \frac{\mathbb E[Yc(Y,D)] - w^T\mathbb E[Xc(Y,D)]}{\mathbb E[c(Y,D)]}.
\end{align*}
Note that $\mathbb E[c(Y,D)] = 1$, so
\[b^* = \mathbb E[Yc(Y,D)] - w^T\mathbb E[Xc(Y,D)].\]
Next, we solve $\partial L(w,b^*)/\partial w = 0$:
\begin{align}
    \frac{\partial L(w,b^*)}{\partial w} &= 2\mathbb E\left[\left(Y - w^TX - \mathbb E[Yc(Y,D)] + w^T\mathbb E[Xc(Y,D)]\right)\left(-X^T +\mathbb E[Xc(Y,D)]^T \right) c(Y,D)\right] = 0 \nonumber\\
    & \implies w^* = \mathbb E\left[(X-\mathbb E[Xc(Y,D)])(X-\mathbb E[Xc(Y,D)])^Tc(Y,D) \right]^{-1} \times \nonumber \\
    & \qquad \qquad \qquad\mathbb E\left[(X-\mathbb E[Xc(Y,D)])(Y-\mathbb E[Yc(Y,D)])^Tc(Y,D) \right].
    \label{eq:opt-weights}
\end{align}
In order to derive the general form for the WGE of a model $f_{\theta}$, we begin by deriving the individual error terms $E^{(y,d)}(f_{\theta})$, $(y,d) \in \{0,1\} \times \{S,T\}$, as follows:

\noindent\begin{minipage}[t]{.5\linewidth}
  \begin{align*}
    E^{(0,d)}(f_{\theta}) &\coloneqq P\left(\mathbbm 1\left\{w^TX+b > 1/2\right\} \ne Y \mid Y=0,D=d\right) \\
    & = P\left(w^TX+b > 1/2 \mid Y=0,D=d\right) \\
    & = 1-\Phi\left(\frac{1/2-(w^T\mu^{(0,d)}+b)}{\sqrt{w^T\Sigma w}} \right) \\
    & = \Phi\left(\frac{w^T\mu^{(0,d)}+b - 1/2}{\sqrt{w^T\Sigma w}} \right),
\end{align*}
\end{minipage}%
\begin{minipage}[t]{.5\linewidth}
  \begin{align*}
    E^{(1,d)}(f_{\theta}) &\coloneqq P\left(\mathbbm 1\left\{w^TX+b > 1/2\right\} \ne Y \mid Y=1,D=d\right) \\
    & = P\left(w^TX+b \le 1/2 \mid Y=1,D=d\right) \\
    & = \Phi\left(\frac{1/2-(w^T\mu^{(1,d)}+b)}{\sqrt{w^T\Sigma w}} \right).
\end{align*}
\end{minipage}

\subsubsection{Upweighting}
Given $c(y,d) = {1}/({4\pi^{(y,d)}})$ for $(y,d) \in \{0,1\}\times\{S,T\}$, we compute the required expectations to find the optimal $w^*_\text{UW}$ in \eqref{eq:opt-weights} as follows:
\begin{align*}
    \mathbb E[c(Y,D)] &= \frac{1}{4} \sum_{(y,d) \in \{0,1\}\times\{S,T\}}\frac{\pi^{(y,d)}}{4\pi^{(y,d)}} = 1, \\
    \mathbb E[Yc(Y,D)] &= \frac{1}{4} \sum_{(y,d) \in \{0,1\}\times\{S,T\}}\frac{\pi^{(y,d)} y}{4\pi^{(y,d)}} = \frac{1}{2}, \\
    \mathbb E[Xc(Y,D)] &= \mathbb E\left[c(Y,D) \mathbb E[X|D,Y]\right] \\
    & = \frac{1}{4} \sum_{(y,d) \in \{0,1\}\times\{S,T\}} \frac{[\mathbbm{1}(d=T)(\mu^{(0,T)}-\mu^{(0,S)})+(1-y)\mu^{(0,S)}+y\mu^{(1,S)}]\pi^{(y,d)}}{\pi^{(y,d)}} \\
    & = \frac{1}{2}(\mu^{(0,T)} + \mu^{(1,S)}), \\
    \mathbb E[XYc(Y,D)]  
    & = \mathbb E\left[c(Y,D)Y \mathbb E[X|D,Y]\right] \\
    & =  \frac{1}{4} \sum_{(y,d) \in \{0,1\}\times\{S,T\}} \frac{[\mathbbm{1}(d=T)(\mu^{(0,T)}-\mu^{(0,S)})+(1-y)\mu^{(0,S)}+y\mu^{(1,S)}]\pi^{(y,d)}y}{\pi^{(y,d)}} \\
    & = \frac{1}{4}(\mu^{(0,T)}-\mu^{(0,S)} + 2\mu^{(1,S)}), \\
    \mathbb E[XX^Tc(Y,D)]  & = \mathbb E\left[c(Y,D) \mathbb E[XX^T|D,Y]\right] \\
    & = \mathbb E\left[c(Y,D) \mathbb E[ \text{Var}(X|D,Y) +\mathbb E[X|D,Y] \mathbb E[X|D,Y]^T] \right] \\
    & = \frac{1}{4}\Big(4\Sigma + 2\Delta_C \Delta_C^T - 2\Delta_C(\mu^{(0,S)})^T - 2\mu^{(0,S)}\Delta_C^T + 4\mu^{(0,S)}(\mu^{(0,S)})^T \\
    & \qquad\qquad - \Delta_C\Delta_D^T - \Delta_D\Delta_C^T + 2\Delta_D\Delta_D^T+ 2\Delta_D(\mu^{(0,S)})^T + 2\mu^{(0,S)}\Delta_D^T \Big) \\
    & \overset{(i)}{=} \frac{1}{4}\Big(4\Sigma + \Delta_C \Delta_C^T - \Delta_C(\mu^{(0,S)}+\mu^{(1,T)})^T - 2\mu^{(0,S)}\Delta_C^T + \Delta_D\Delta_D^T \\
    & \qquad\qquad + 4\mu^{(0,S)}(\mu^{(0,S)})^T+ \Delta_D(\mu^{(0,S)}+\mu^{(1,T)})^T + 2\mu^{(0,S)}\Delta_D^T \Big) \\
    & \overset{(ii)}{=} \frac{1}{4}\Big(4\Sigma + \Delta_C \Delta_C^T - \Delta_C(\mu^{(0,S)}+\mu^{(1,T)})^T + 2\mu^{(0,S)}(\mu^{(0,T)})^T \\
    & \qquad\qquad + \Delta_D\Delta_D^T + \Delta_D(\mu^{(0,S)}+\mu^{(1,T)})^T + 2\mu^{(0,S)}(\mu^{(1,S)})^T \Big) \\
    & = \frac{1}{4}\Big(4\Sigma + \Delta_C \Delta_C^T + \Delta_D \Delta_D^T+ (\Delta_D - \Delta_C + 2\mu^{(0,S)})(\mu^{(0,T)} + \mu^{(1,S)})^T\Big) \\
    & \overset{(iii)}{=} \frac{1}{4}\Big(4\Sigma + \Delta_C \Delta_C^T + \Delta_D \Delta_D^T+ (\mu^{(0,T)} + \mu^{(1,S)})(\mu^{(0,T)} + \mu^{(1,S)})^T\Big),
\end{align*}
where $(i)$ follows from substituting 
\begin{align*}
    -2\Delta_C(\mu^{(0,S)})^T&=\Delta_C\Delta_D^T-\Delta_C\Delta_C^T-\Delta_C(\mu^{(0,T)}+\mu^{(1,S)})^T, \\
    2\Delta_D(\mu^{(0,S)})^T&=\Delta_D\Delta_C^T-\Delta_D\Delta_D^T+\Delta_D(\mu^{(0,T)}+\mu^{(1,S)})^T,
\end{align*}
$(ii)$ follows from substituting $\Delta_C = \mu^{(0,S)} - \mu^{(0,T)}$ and $\Delta_D = \mu^{(1,S)}-\mu^{(0,S)}$, and $(iii)$ follows from substituting $\mu^{(0,S)}+\mu^{(1,T)}=\mu^{(1,S)}+\mu^{(0,T)}$.
Substituting the above expectations into \eqref{eq:opt-weights} yields
\begin{align*}
    \mathbb E\left[(X-\mathbb E[Xc(Y,D)])(X-\mathbb E[Xc(Y,D)])^Tc(Y,D) \right] & = \mathbb E[XX^Tc(Y,D)] - \mathbb E[Xc(Y,D)]\mathbb E[Xc(Y,D)]^T \\
    & = \Sigma + \frac{1}{4}\Delta_C \Delta_C^T + \frac{1}{4}\Delta_D \Delta_D^T,
\end{align*}
and
\begin{align*}
    \mathbb  E\left[(X-\mathbb E[Xc(Y,D)])(Y-\mathbb E[Yc(Y,D)])^Tc(Y,D) \right]  = \mathbb E[XYc(Y,D)] - \mathbb E[Xc(Y,D)]\mathbb E[Yc(Y,D)] 
     = \frac{1}{4}\Delta_D.
\end{align*}
Therefore,
\begin{equation}
    w^*_\text{UW} = \frac{1}{4}\left(\Sigma + \frac{1}{4}\Delta_C \Delta_C^T + \frac{1}{4}\Delta_D \Delta_D^T\right)^{-1}\Delta_D \quad \text{and} \quad  b^*_\text{UW} = \frac{1}{2}-\frac{1}{2}(w^*_\text{UW})^T\left(\mu^{(0,T)}+\mu^{(1,S)}\right),
    \label{eq:opt-w-uw}
\end{equation}
and hence
\begin{align}
    E^{(0,d)}(f_{\theta_\text{UW}^*})
    = \Phi\left(\frac{(w_\text{UW}^*)^T\left(\mu^{(0,d)}-\frac{1}{2}(\mu^{(0,T)}-\mu^{(1,S)})\right)}{\sqrt{(w_\text{UW}^*)^T\Sigma w_\text{UW}^*}} \right),  E^{(1,d)}(f_{\theta_\text{UW}^*})
    = \Phi\left(\frac{-(w^*_\text{UW})^T\left(\mu^{(1,d)}-\frac{1}{2}(\mu^{(0,T)}-\mu^{(1,S)})\right)}{\sqrt{(w^*_\text{UW})^T\Sigma w^*_\text{UW}}} \right).
    \label{eq:misclassification-error-terms-uw}
\end{align}
We can simplify the expressions in \eqref{eq:misclassification-error-terms-uw} by using the following relations:
\begin{align}
    \mu^{(0,T)}-\frac{1}{2}(\mu^{(0,T)}-\mu^{(1,S)}) &= \frac{1}{2}(\mu^{(0,T)} - \mu^{(1,S)}) = \frac{1}{2}(\mu^{(0,T)} - \mu^{(1,T)} + \mu^{(1,T)} - \mu^{(1,S)}) = -\frac{1}{2}(\Delta_C + \Delta_D), \\
    \mu^{(0,S)}-\frac{1}{2}(\mu^{(0,T)}-\mu^{(1,S)}) &= \frac{1}{2}(\mu^{(0,S)} - \mu^{(0,T)}) + \frac{1}{2}(\mu^{(0,S)} - \mu^{(1,S)})  = \frac{1}{2}(\Delta_C - \Delta_D), \\
    \mu^{(1,S)}-\frac{1}{2}(\mu^{(0,T)}-\mu^{(1,S)}) &= \frac{1}{2}(\mu^{(1,S)} - \mu^{(0,T)}) = \frac{1}{2}(\mu^{(1,S)} - \mu^{(1,T)} + \mu^{(1,T)} - \mu^{(0,T)})  = \frac{1}{2}(\Delta_C + \Delta_D), \\
    \mu^{(1,T)}-\frac{1}{2}(\mu^{(0,T)}-\mu^{(1,S)}) &= \frac{1}{2}(\mu^{(1,T)} - \mu^{(0,T)}) + \frac{1}{2}(\mu^{(1,T)} - \mu^{(1,S)})  = \frac{1}{2}(\Delta_D - \Delta_C). 
\end{align}
Plugging these into \eqref{eq:misclassification-error-terms-uw} for each group $(y,d) \in \{0,1\} \times \{S,T\}$ yields
\[E^{(0,T)}(f_{\theta_\text{UW}^*})
    = E^{(1,S)}(f_{\theta_\text{UW}^*}) = \Phi\left(\frac{-\frac{1}{2}(w_\text{UW}^*)^T\left(\Delta_C+\Delta_D\right)}{\sqrt{(w_\text{UW}^*)^T\Sigma w_\text{UW}^*}} \right), E^{(0,S)}(f_{\theta_\text{UW}^*})
    = E^{(1,T)}(f_{\theta_\text{UW}^*}) = \Phi\left(\frac{\frac{1}{2}(w_\text{UW}^*)^T\left(\Delta_C-\Delta_D\right)}{\sqrt{(w_\text{UW}^*)^T\Sigma w_\text{UW}^*}} \right).\]
Thus,
\begin{align}
    \text{WGE}(f_{\theta_\text{UW}^*}) = \max\left\{\Phi\left(\frac{-\frac{1}{2}(w_\text{UW}^*)^T\left(\Delta_C+\Delta_D\right)}{\sqrt{(w_\text{UW}^*)^T\Sigma w_\text{UW}^*}} \right), \Phi\left(\frac{\frac{1}{2}(w_\text{UW}^*)^T\left(\Delta_C-\Delta_D\right)}{\sqrt{(w_\text{UW}^*)^T\Sigma w_\text{UW}^*}} \right) \right\}
\end{align}
\subsubsection{Downsampling}
Since DS is a special case of SRM with $\pi_0=1/4$, we first derive the optimal model parameters for SRM for any $\pi_0 \le 1/4$. In the case of SRM (and therefore DS), $c(y,d)=1$ for $(y,d)\in\mathcal{Y}\times\mathcal{D}$, so \eqref{eq:opt-weights} simplifies to
\begin{equation}
    w_\text{SRM}^* = \text{Var}(X)^{-1}\text{Cov}(X,Y).
\end{equation}
Let $\pi^{(d|y)} \coloneqq P(D=d|Y=y)$ for $(y,d) \in \mathcal{Y}\times\mathcal{D}$, $\mu^{(y)}\coloneqq\mathbb E[X|Y=y]$ for $y \in \mathcal{Y}$ and $\bar\Delta\coloneqq \mu^{(1)}-\mu^{(0)}$. We compute $\text{Var}(X)$ as follows:
\begin{align}
    \text{Var}(X) &= \mathbb E[\text{Var}(X|Y)] + \text{Var}(\mathbb E[X|Y]) \\
    & = \mathbb E\left[\mathbb E[\text{Var}(X|Y,D)|Y] + \text{Var}(\mathbb E[X|Y,D]|Y)\right] + \text{Var}(\mathbb E[X|Y]) \\
    & = \Sigma + \mathbb E[\text{Var}(\mathbb E[X|Y,D]|Y)] + \text{Var}(\mathbb E[X|Y]) \\
    & =  \Sigma + \mathbb E[\text{Var}(\mathbbm{1}(D=T)(\mu^{(0,T)}-\mu^{(0,S)})+(1-Y)\mu^{(0,S)}+Y\mu^{(1,S)}|Y)] + \text{Var}(\mathbb E[X|Y]) \\
    & = \Sigma + \Delta_C\Delta_C^T\mathbb E[\text{Var}(\mathbbm{1}(D=T)|Y)] + \text{Var}(Y(\mu^{(1)}-\mu^{(0)})+\mu^{(0)}) \\
    & = \Sigma + \Delta_C\Delta_C^T\mathbb E[\text{Var}(D|Y)] + \bar\Delta\bar\Delta^T\text{Var}(Y) \\
    & = \Sigma + \Delta_C\Delta_C^T\mathbb E[Y\pi^{(R|1)}\pi^{(B|1)}+(1-Y)\pi^{(R|0)}\pi^{(B|0)}] + \bar\Delta\bar\Delta^T\pi^{(1)}\pi^{(0)} \\
    & = \Sigma + 2\pi_0(1-2\pi_0)\Delta_C\Delta_C^T + \frac{1}{4}\bar\Delta\bar\Delta^T.
    \label{eq:erm-var}
\end{align}
Next, we compute $\text{Cov}(X,Y)$ as follows:
\begin{align}
    \text{Cov}(X,Y) &= \mathbb E[\text{Cov}(X,Y|Y)] + \text{Cov}(\mathbb E[X|Y], \mathbb E[Y|Y]) \\
    & = \text{Cov}(\mathbb E[X|Y], Y) \\
    & = \text{Cov}(\mu^{(0)}+Y\bar\Delta, Y) \\
    & = \text{Cov}(Y\bar\Delta, Y) \\
    & = \bar\Delta\text{Var}(Y) \\
    & = \pi^{(1)}\pi^{(0)}\bar\Delta \\
     & = \frac{1}{4}\bar\Delta.
     \label{eq:erm-cov}
\end{align}

\noindent In order to write \eqref{eq:erm-var} and \eqref{eq:erm-cov} only in terms of $\Delta_C$ and $\Delta_D$ and to see the effect of $\pi_0$, we show the following relationship between $\bar\Delta$, $\Delta_C$ and $\Delta_D$.
\begin{lemma}
\label{lemma:triangle}
    Let $\bar\Delta\coloneqq \mu^{(1)}-\mu^{(0)}$. Then 
    \begin{align*}
        \Delta_D - \bar\Delta = (1-4\pi_0)\Delta_C
    \end{align*}
\end{lemma}

\begin{proof}
    We first note that \[
        \mu^{(1)} = 2\pi_0(\mu^{(1,S)}-\mu^{(1,T)}) + \mu^{(1,T)} = 2\pi_0 \Delta_C + \mu^{(1,T)}.
    \] Similarly, \[
        \mu^{(0)} = 2\pi_0(\mu^{(0,T)}-\mu^{(0,S)}) + \mu^{(0,S)} = -2\pi_0 \Delta_C + \mu^{(0,S)}.
    \]
    Combining these with the definitions of $\Delta_D$ and $\bar\Delta$, we get \begin{align*}
        \Delta_D - \bar\Delta &= \Delta_D - (\mu^{(1)} - \mu^{(0)}) \\ 
        &= \mu^{(1,T)} - \mu^{(0,T)} -2\pi_0 \Delta_C - \mu^{(1,T)} - 2\pi_0 \Delta_C + \mu^{(0,S)} \\
        &= \mu^{(0,S)} - \mu^{(0,T)} - 4\pi_0 \Delta_C\\
        &= (1-4\pi_0)\Delta_C.
    \end{align*}
\end{proof}
\noindent From \eqref{eq:erm-var}, \eqref{eq:erm-cov} and \cref{lemma:triangle}, we then obtain
\begin{equation}
    w^*_\text{SRM} = \frac{1}{4}\left(\Sigma + 2\pi_0(1-2\pi_0)\Delta_C\Delta_C^T + \frac{1}{4}\bar\Delta\bar\Delta^T \right)^{-1}\bar\Delta,
    \label{eq:opt-w-srm}
\end{equation}
where $\bar\Delta \coloneqq \mu^{(1)}-\mu^{(0)} = \Delta_D - (1-4\pi_0)\Delta_C$, and
\begin{equation}
    b^*_\text{SRM} = \frac{1}{2} - \frac{1}{2}(w^*_\text{SRM})^T(\mu^{(0,T)}+\mu^{(1,S)}).
\end{equation}
When $\pi_0 = 1/4$, $\bar\Delta=\Delta_D$, and therefore,
\begin{equation}
    w^*_\text{DS} = \frac{1}{4}\left(\Sigma + \frac{1}{4}\Delta_C\Delta_C^T + \frac{1}{4}\Delta_D\Delta_D^T \right)^{-1}\Delta_D \quad \text{and} \quad b^*_\text{DS} = \frac{1}{2} - \frac{1}{2}(w^*_\text{DS})^T\left(\mu^{(0,T)}+\mu^{(1,S)}\right).
\end{equation}

Since $w^*_{DS}=w^*_{UW}$, we have that $\text{WGE}(f_{\theta_\text{DS}^*}) = \text{WGE}(f_{\theta_\text{UW}^*})$.



\subsection{Proof of \cref{thm:wge-comparison}}
\label{appendix:wge-comparison}

In the following, we will be using the following lemma.
\begin{lemma}
    Let $A \in \mathbb R^{p\times p}$ be symmetric positive definite (SPD) and $u,v \in \mathbb R^p$. Then
    \[(A + vv^T + uu^T)^{-1}u = c_u\left( A^{-1}u - c_v A^{-1}v\right) \quad \text{with} \quad c_u \coloneqq \frac{1}{1+u^TB^{-1}u} \quad \text{and} \quad c_v \coloneqq\frac{v^TA^{-1}u}{1+v^TA^{-1}v}.\]
\label{lemma:sherman-morrison-deriv}
\end{lemma}
\begin{proof}
    Let $B \coloneqq A + vv^T$. Then
\begin{align*}
    (A + vv^T + uu^T)^{-1}u & = (B+uu^T)^{-1}u \\
    & = \left(B^{-1}- \frac{B^{-1}uu^TB^{-1}}{1+u^TB^{-1}u}\right)u \qquad \text{(Sherman-Morrison formula)} \\
    & = B^{-1}u- \frac{u^TB^{-1}u}{1+u^TB^{-1}u}B^{-1}u \\
    & = c_uB^{-1}u \\
    & = c_u\left( A^{-1} - \frac{A^{-1}vv^TA^{-1}}{1+v^TA^{-1}v}\right)u \qquad \text{(Sherman-Morrison formula)}.
\end{align*}
The assumption that $A$ is SPD guarantees that $A^{-1}$ exists, $B$ is SPD, and $c_u$ and $c_v$ are well-defined.
\end{proof}
\subsubsection{SRM}
From \eqref{eq:opt-w-srm}, we have that 
\begin{equation}
    w^*_\text{SRM} = \frac{1}{4}\left(\Sigma + 2\pi_0(1-2\pi_0)\Delta_C\Delta_C^T + \frac{1}{4}\bar\Delta\bar\Delta^T \right)^{-1}\bar\Delta.
    \label{eq:w-srm-thm2}
\end{equation}
Applying \cref{lemma:sherman-morrison-deriv} to \eqref{eq:w-srm-thm2} with $A = \Sigma$, $u = \bar\Delta/2$ and $v = \sqrt{\beta} \Delta_C$, where $\beta \coloneqq 2\pi_0(1-2\pi_0)$, yields
\begin{align*}
    w^*_\text{SRM} & = \gamma_\text{SRM}\left(\Sigma^{-1}\bar\Delta-\frac{\beta\Delta_C^T\Sigma^{-1}\bar\Delta}{1+\beta\Delta_C^T\Sigma^{-1}\Delta_C}\Sigma^{-1}\Delta_C \right) \\
    & = \gamma_\text{SRM}\left(\Sigma^{-1}\Delta_D- \frac{1-4\pi_0+\beta\Delta_C^T\Sigma^{-1}\Delta_D}{1+\beta\Delta_C^T\Sigma^{-1}\Delta_C}\Sigma^{-1}\Delta_C \right) \qquad \text{(substituting $\bar\Delta = \Delta_D - (1-4\pi_0)\Delta_C$)} \\
    & = \gamma_\text{SRM}\left(\Sigma^{-1}\Delta_D-c_{\pi_0}\Sigma^{-1}\Delta_C \right)
\end{align*}
with
\[\gamma_\text{SRM} \coloneqq \frac{1}{4+\bar\Delta^TA_\text{SRM}^{-1}\bar\Delta} \quad \text{and} \quad c_{\pi_0} \coloneqq \frac{1-4\pi_0+\beta\Delta_C^T\Sigma^{-1}\Delta_D}{1+\beta\Delta_C^T\Sigma^{-1}\Delta_C}.\] 
Let $\|v \| \coloneqq \sqrt{v^T\Sigma^{-1}v}$ be the norm induced by the $\Sigma^{-1}$--inner product.
Then
\begin{align*}
    \text{WGE}(f_{\theta_\text{SRM}^*}) = \max\left\{\Phi\left(\frac{(c_{\pi_0}-1)\Delta_C^T\Sigma^{-1}\Delta_D - \|\Delta_D \|^2 + c_{\pi_0}\|\Delta_C \|^2}{2\|{\Delta_D - c_{\pi_0} \Delta_C}\|} \right), \Phi\left(\frac{(c_{\pi_0}+1)\Delta_C^T\Sigma^{-1}\Delta_D - \|\Delta_D \|^2 - c_{\pi_0}\|\Delta_C \|^2}{2\|\Delta_D - c_{\pi_0} \Delta_C\|} \right) \right\}
\end{align*}
Under \cref{as:orthogonality}, $c_{\pi_0} = \Tilde{c}_{\pi_0} \coloneqq ({1-4\pi_0})/({1+2\pi_0(1-2\pi_0)\|\Delta_C\|^2})$ and
\begin{align}
    \text{WGE}(f_{\theta_\text{SRM}^*}) = \max\left\{\Phi\left(\frac{- \|\Delta_D \|^2 + \Tilde{c}_{\pi_0}\|\Delta_C \|^2}{2\sqrt{\|{\Delta_D}\|^2 + \Tilde{c}_{\pi_0}^2 \|{\Delta_C}\|^2}} \right), \Phi\left(\frac{ - \|\Delta_D \|^2 - \Tilde{c}_{\pi_0}\|\Delta_C \|^2}{2\sqrt{\|{\Delta_D}\|^2 + \Tilde{c}_{\pi_0}^2 \|{\Delta_C}\|^2}} \right) \right\},
    \label{eq:ERM_error}
\end{align}
where the first term is the error of the minority groups and the second is that of the majority groups.
In order to compare the WGE of SRM with the WGEs of the other methods, we first show that under \cref{as:orthogonality} the WGE of SRM is given by the minority error term in \eqref{eq:ERM_error}. Since $\Tilde{c}_{\pi_0} \ge 0$ for $\pi_0 \le 1/4$ with equality at $\pi_0=1/4$ , we have that
\begin{align*}
    \frac{- \|\Delta_D \|^2 + \Tilde{c}_{\pi_0}\|\Delta_C \|^2}{2\sqrt{\snorm{\Delta_D}^2 + \Tilde{c}_{\pi_0}^2 \snorm{\Delta_C}^2}} \ge \frac{ - \|\Delta_D \|^2 - \Tilde{c}_{\pi_0}\|\Delta_C \|^2}{2\sqrt{\snorm{\Delta_D}^2 + \Tilde{c}_{\pi_0}^2 \snorm{\Delta_C}^2}}
    \quad \Leftrightarrow \quad\Tilde{c}_{\pi_0}\snorm{\Delta_C}^2 \ge 0,
\end{align*}
which is satisfied for all $\pi_0 \le 1/4$ with equality at $\pi_0 = 1/4$. Since $\Phi$ is increasing, we get
\begin{align}
    \text{WGE}(f_{\theta_\text{SRM}^*}) = \Phi\left(\frac{- \|\Delta_D \|^2 + \Tilde{c}_{\pi_0}\|\Delta_C \|^2}{2\sqrt{\snorm{\Delta_D}^2 + \Tilde{c}_{\pi_0}^2 \snorm{\Delta_C}^2}} \right).
    \label{eq:ERM_error_orth}
\end{align}
\subsubsection{Downsampling}

Since DS is a special case of SRM with $\pi_0=1/4$, we only need to examine \eqref{eq:ERM_error_orth} as a function of $\pi_0$. Still under \cref{as:orthogonality}, we take the following derivative:
\begin{align*}
    \frac{\partial}{\partial \pi_0} \frac{- \|\Delta_D \|^2 + \Tilde{c}_{\pi_0}\|\Delta_C \|^2}{2\sqrt{\snorm{\Delta_D}^2 + \Tilde{c}_{\pi_0}^2 \snorm{\Delta_C}^2}} & = \frac{\snorm{\Delta_D}^2\snorm{\Delta_C}^2(\Tilde{c}_{\pi_0}+1)}{2(\snorm{\Delta_D}^2+\Tilde{c}_{\pi_0}^2 \snorm{\Delta_C}^2)^{3/2}}*\frac{-2(16\pi_0^2-8\pi_0+3)\snorm{\Delta_C}}{(1+2\pi_0(1-2\pi_0)\snorm{\Delta_C}^2)^2},
\end{align*}
    which is strictly negative for $\pi_0 \le 1/4$. Thus, the WGE of SRM is strictly decreasing as a function of $\pi_0$, i.e., for $\pi_0 < 1/4$,
    \begin{align}
    \text{WGE}(f_{\theta_\text{SRM}^*}) > \text{WGE}(f_{\theta_\text{DS}^*}) = \Phi\left(-\|\Delta_D \|/2  \right).
    \label{eq:ERM_DS_comp-1}
    \end{align}
\subsubsection{Upweighting}
From the proof of \cref{prop:ds-uw-weights-eq} (\cref{appendix:ds-uw-weights-eq}), we have that
\begin{align}
     \text{WGE}(f_{\theta_\text{UW}^*}) = \text{WGE}(f_{\theta_\text{DS}^*}) = \Phi\left(-\|\Delta_D \|/2  \right).
    \label{eq:ERM_DS_comp}
\end{align}
\subsubsection{Intra-class domain mixup}
From \cite{yaoImprovingOutofDistributionRobustness2022}, we have
\begin{equation}
    w_\text{MU}^* = \frac{1}{4}\left(2\mathbb E[\Lambda^2]\Sigma + \text{Var}(\Lambda)\Delta_C\Delta_C^T + \frac{1}{4}\Delta_D\Delta_D^T \right)^{-1}\Delta_D.
\end{equation}
Applying \cref{lemma:sherman-morrison-deriv} to \eqref{eq:w-srm-thm2} with $A = 2\mathbb E[\Lambda^2]\Sigma$, $u = \Delta_D/2$ and $v = \sqrt{\text{Var}(\Lambda)} \Delta_C$ yields
\begin{align*}
    w^*_\text{MU} =\gamma_\text{MU}\left(\Sigma^{-1}\Delta_D-c_\text{MU}\Sigma^{-1}\Delta_C \right) \text{ with }  \gamma_\text{MU} \coloneqq \frac{1}{2\mathbb E[\Lambda^2]}\left( \frac{1}{4+\Delta_D^TA_\text{MU}^{-1}\Delta_D}\right) \text{ and } c_\text{MU} \coloneqq \frac{\text{Var}(\Lambda)\Delta_C^T\Sigma^{-1}\Delta_D}{2\mathbb E[\Lambda^2]+\text{Var}(\Lambda)\Delta_C^T\Sigma^{-1}\Delta_C},
\end{align*}
and
\begin{align*}
    b^*_\text{MU} &=\frac{1}{2}-(w^*_\text{MU})^T\left[(\Delta_D+\Delta_C)/2+\mu^{(0,T)} \right] = \frac{1}{2}-(w^*_\text{MU})^T(\mu^{(0,T)} + \mu^{(1,S)}).
\end{align*}
Under \cref{as:orthogonality}, $c_\text{MU} = 0$, so $w^*_\text{MU} =\gamma_\text{MU}\left(\Sigma^{-1}\Delta_D\right)$,
and therefore,
\begin{align}
         \text{WGE}(f_{\theta_\text{MU}^*}) = \Phi\left(-\|\Delta_D \|/2  \right).
\end{align}
\fi
\end{document}